\documentclass{article}

\usepackage[preprint]{neurips_2026}

\usepackage[utf8]{inputenc}
\usepackage[T1]{fontenc}
\usepackage{hyperref}
\usepackage{url}
\usepackage{booktabs}
\usepackage{amsfonts}
\usepackage{nicefrac}
\usepackage{microtype}
\usepackage{xcolor}
\usepackage{graphicx}
\usepackage{subcaption}
\usepackage{multirow}
\usepackage{tikz}
\usetikzlibrary{arrows.meta}
\usepackage{amsmath}
\usepackage{amssymb}
\usepackage{mathtools}
\usepackage{amsthm}
\usepackage{algorithm}
\usepackage{algpseudocode}

\theoremstyle{plain}
\newtheorem{proposition}{Proposition}

\theoremstyle{remark}
\newtheorem{remark}{Remark}

\title{\Large Tippett-minimum Fusion of Representation-space Diffusion Models
  for Multi-Encoder Out-of-Distribution Detection}

\author{%
  Neelkamal Bhuyan \\
  Georgia Institute of Technology \\
  \texttt{nbhuyan3@gatech.edu}
}

\begin{document}

\maketitle

\begin{abstract}
We address out-of-distribution (OOD) detection across the full spectrum of
distribution shifts---global domain changes, semantic divergence, texture differences,
and covariate corruptions---through a multi-encoder fusion of per-encoder representation-space
diffusion models (RDMs).
We statistically identify each encoder's sensitivity to specific shift types from ID
data alone and introduce EncMin2L---an encoder-agnostic two-level $\min(\cdot)$-gate that
combines and calibrates per-encoder diffusion-based likelihood detectors without OOD
labels, outperforming monolithic multi-encoder baselines at $2.3\!\times$ lower
parameter cost.
Two ID-data diagnostics: $\eta^2$ (class-conditional F-test) and $\Delta\mu$
(log-likelihood shift under synthetic corruptions)---quantify encoder specialization,
while a Tippett minimum $p$-value combination aggregates per-encoder scores into a
single, calibration-stable OOD signal.
EncMin2L achieves $\geq\!0.94$ AUROC across all four shift types simultaneously,
outperforming the state-of-the-art representation-space diffusion OOD detectors across overlapping benchmarks.
\end{abstract}

% ============================================================
\section{Introduction}
\label{sec:intro}
% ============================================================

Out-of-distribution (OOD) detection asks whether a given image was generated by the
same process as the true data, often known as the in-distribution (ID) data.
A standard approach fits a density model to ID training data
and flags a test point as OOD when its log-likelihood $\log p_\theta(z)$
falls below a threshold calibrated on held-out ID samples~\citep{xiao2020likelihood}.
The challenge lies in building density models expressive enough to capture the full
structure of the ID distribution yet sensitive to the specific forms of shift that
arise in practice.

Deep generative models have become the dominant framework for learning complex data
distributions, with diffusion models demonstrating remarkable fidelity across image
generation, density estimation, and downstream tasks including anomaly detection and
domain adaptation~\cite{ho2020denoising, livernoche2024diffusion}.
Score-matching stochastic differential equations
(SDEs)~\citep{song2021score} provide a principled framework: a neural network
learns the score function $\nabla_z \log p(z)$ via a denoising objective, and the
associated probability flow ODE (PF-ODE) enables numerically accurate computation
of $\log p_\theta(z)$ for any test point~\cite{song2021maximum}.
This makes score-based diffusion models natural candidates for likelihood-based OOD
detection.

Operating diffusion models directly in pixel space is computationally expensive,
and pixel-space density estimates are notoriously unreliable for
OOD~\citep{xiao2020likelihood, graham2021denoising}.
Representation-space diffusion models (RDMs) address this by operating in the compressed
representation space of a pretrained encoder---capturing underlying semantic
structure by removing pixel-level variation~\cite{rombach2022high}.
Modern self-supervised encoders are each capable of meaningfully compressing
$224\!\times\!224$ images into dense semantic vectors: CLIP~\citep{radford2021clip}
encodes global visual-linguistic alignment; DINOv2~\citep{oquab2023dinov2} captures
fine-grained patch-level semantics; and ResNet-50~\citep{he2016resnet} penultimate
features encode texture and low-level image statistics.

This compression is not lossless---each encoder emphasises a different aspect of
image structure and is therefore sensitive to a different class of distribution shift.
In \S\ref{sec:analysis} we quantify this directly: CLIP detects global domain
shift but is nearly blind to covariate corruptions; DINOv2 separates fine-grained
semantic classes but its features are largely invariant to image corruptions;
ResNet-50 is the strongest covariate-shift detector but ranks last on semantic
benchmarks.
This raises a fundamental question:
\begin{quote}
  \textit{Given a library of pretrained encoders, which
subset should be deployed, and how should their signals be combined, when the type
of distribution shift is unknown at test time?}
\end{quote}

\noindent \textbf{Main Contributions}: We make two contributions that together answer this question.
\textbf{(i)~Quantifying encoder sensitivity from ID data.}
We quantify each encoder's OOD sensitivity through ID-data diagnostics ($\eta^2$,
$\Delta\mu$, and cross-encoder Spearman $\rho$) and characterise how representation
space geometry governs sensitivity to covariate versus semantic distribution shifts.
These diagnostics, computed entirely from ID training data, reveal strict encoder
specialisation: no encoder dominates across all shift types simultaneously, and any
fixed weighting of encoder contributions will underperform on at least one shift type.
\textbf{(ii)~Two-level $\min(\cdot)$-gate with statistical calibration.}
We combine per-encoder VP-SDE density models through a two-level $\min(\cdot)$-gate grounded
in Tippett's minimum $p$-value test, resulting in our \textbf{EncMin2L} framework.
At each level, raw log-likelihood scores are first mapped through the ID empirical
CDF to yield calibrated $p$-values on $[0,1]$.
For every encoder, a $\min(\cdot)$ operation across $n$ different input-normalization forks collapses $n$ complementary
views of each encoder's representation space into a single per-encoder score
(Level~1), followed by an explicit CDF-based re-scaling that corrects for the
non-uniform Beta$(1,n)$ null distribution of the resulting minimization operation.
A second Tippett minimization across all encoders' Level~1 scores, with a joint threshold $\tau$
calibrated on held-out ID data, constitutes the final OOD decision (Level~2).
The result is a statistically valid multiple testing procedure in which the
input sample itself forces the most-alarmed representation-space diffusion OOD detector to make a decision---adapting to the
input without requiring any knowledge of the distribution shift at test time.

We evaluate on four OOD benchmarks spanning global domain shift (SVHN),
fine-grained semantic shift (CIFAR-100), texture shift (DTD), and covariate
shift (CIFAR-10C).
EncMin2L achieves $\geq$0.94 AUROC across all four benchmarks, outperforming
monolithic multi-encoder baselines at 2.3$\times$ lower
parameter cost.
Leave-one-out ablations confirm that each encoder specialises on the shift type
predicted by its ID diagnostic profile, providing post-hoc validation of the
framework.

% ============================================================
\section{Related Work}
\label{sec:related}
% ============================================================

\paragraph{Density-based OOD detection.}
Using a generative model's log-likelihood as an OOD score is a longstanding approach. \citet{grathwohl2020jem} reformulate a classifier as an energy-based model;
\citet{xiao2020likelihood} propose likelihood regret for VAEs to correct for
the pathological behaviour of deep generative model likelihoods on OOD inputs.
\citet{graham2021denoising} apply diffusion models directly in pixel space via
multi-step reconstruction error; \citet{liu2023lmd} score checkerboard-masked
inpainting with LPIPS (LMD); \citet{heng2024diffpath} extract a compact
6-dimensional descriptor from the denoising trajectory of a pixel-space DM
(DiffPath), demonstrating zero-shot cross-distribution generalization but at the
cost of coarse OOD resolution.
Our work instead trains density models in the compressed representation space
of pretrained encoders, where the density is better-behaved and
computationally tractable.
\citet{jarve2025vdm} independently pursue this direction: their Variational
Diffusion Model (VDM) trains a learnable-schedule representation-space diffusion model on a
single ResNet-18 encoder and achieves strong AUROC on CIFAR-10 ID benchmarks.
Concurrently, \citet{ding2025rdm} train a sub-VP SDE score model on a single
pretrained encoder's representation space at ImageNet scale, with a
class-conditioned variant (ConRDM) that matches supervised detectors when labels
are available.
Score-based SDE models~\citep{song2021score} provide numerically accurate
log-likelihoods via the probability flow ODE.

\paragraph{Feature-space methods.}
A complementary family of methods detects OOD samples by their position in the
feature space of a pretrained or jointly trained network, without fitting an
explicit density.
\citet{lee2018mahalanobis} use the Mahalanobis distance to class-conditional
Gaussian fits in feature space.
\citet{liu2020energy} exploit the energy function implicit in classifier logits.
\citet{sun2022knn} show that $k$-NN distance in the penultimate feature space
outperforms many parametric alternatives.
These methods depend on a single representation space; combining multiple encoders
requires an explicit fusion strategy.

\paragraph{Multi-encoder and ensemble detection.}
Ensemble and multi-view approaches to OOD detection are relatively unexplored.
\citet{shalev2018ood} show, in a text setting, that replacing a single label embedding
with multiple dense semantic representations from different corpora or architectures
improves OOD discrimination---an early demonstration that representation diversity
across sources provides complementary detection signal.
Prior works that combine multiple representation spaces in vision typically do so by
concatenation or majority voting, without a theoretical account
of when combination helps and when it is redundant.

\paragraph{Multiple hypothesis testing.}
The minimum $p$-value statistic for combining independent tests is classical:
Tippett's test rejects the global null when $\min_k p_k < \tau$, with
$\tau$ set to control the family-wise error rate~\citep{westfall1993resampling}.
Fisher's method~\citep{westfall1993resampling} provides an alternative via
$-2\sum_k\log p_k \sim \chi^2_{2K}$, which is more powerful when multiple tests
simultaneously reject; Tippett's test (and our $\min(\cdot)$-gate) is more powerful when
a single test dominates.
Given the encoder specialization we observe---each encoder covers a distinct
shift type---Tippett's approach is the natural choice: one encoder typically
dominates for any given OOD shift type.

% ============================================================
\section{Preliminaries}
\label{sec:method}
% ============================================================

\subsection{Representation Space and Notation}
\label{sec:notation}

Let $\{E_1,\ldots,E_K\}$ be a fixed set of pretrained image encoders,
$E_k:\mathcal{X}\to\mathbb{R}^{d_k}$.
We use $K\!=\!3$: CLIP ViT-B/32 ($d_1\!=\!512$), DINOv2 ViT-B/14 ($d_2\!=\!768$),
and ResNet-50 penultimate layer ($d_3\!=\!2048$).
For each encoder we define $n=2$ \emph{normalization forks}:
\begin{align*}
  z_{k,\mathrm{n}} &= (E_k(x) - \mu_k)/\sigma_k, &
  z_{k,\mathrm{u}} &= E_k(x),
\end{align*}
where $\mu_k, \sigma_k$ are per-dimension statistics estimated on the ID training
set only.
Normalization amplifies fine-grained structure within each encoder's distribution;
the unnormed fork retains the raw geometry.

\begin{remark}
  \textit{Within each encoder, the normed fork is more sensitive to covariate shifts and
the unnormed fork holds a marginal edge on semantic OOD---a consequence of
per-dimension Z-score normalization expanding the dynamic range of score
differences under perturbations.}
\end{remark}

\subsection{Score Network Architecture and Per-Encoder Density Models}
\label{sec:tests}

\paragraph{Per-encoder models.}
For each fork $(k, f)$ we train an independent VP-SDE score
network~\citep{song2021score} with linear noise schedule
$\beta(t) = \beta_\mathrm{min} + t(\beta_\mathrm{max}-\beta_\mathrm{min})$,
$\beta_\mathrm{min}=0.1$, $\beta_\mathrm{max}=20$.
The score network is an MLP with hidden dimension $H_k=2d_k$ and depth
$L=6$ residual layers---capacity proportional to the encoder's intrinsic
complexity.
All models are trained with Adam~\citep[lr=$2\!\times\!10^{-4}$, up to 500
epochs with early stopping]{song2021score}, batch size 512, gradient clipping
(max norm 1.0).
The probability flow ODE (PF-ODE) corresponding to the VP-SDE provides
an exact log-likelihood via the instantaneous change-of-variables formula
and Hutchinson trace estimation with $M$ random probes:
\begin{equation}
  \ell_{k,f}(z) \triangleq \log p_\theta^{(k,f)}(z)
    = \log p_T(\tilde{z}) - \int_0^T \nabla \cdot f_\theta(\tilde{z}(t),t)\,dt,
  \label{eq:ll}
\end{equation}
where $f_\theta$ is the PF-ODE drift and $p_T$ is the terminal Gaussian prior.
We use $M\!=\!10$ Rademacher probe vectors and an adaptive ODE solver with
relative and absolute tolerances of $10^{-5}$.
Let $\hat{F}_{k,f}$ be the empirical CDF of $\{\ell_{k,f}(z_i)\}$ over ID
validation samples.
The transformed score
\begin{equation}
  r_{k,f}(z) = \hat{F}_{k,f}\!\left(\ell_{k,f}(z)\right)
  \label{eq:pvalue}
\end{equation}
is a $p$-value: $r_{k,f}(z)\approx\mathrm{Uniform}(0,1)$ for ID samples by
the probability integral transform; small values indicate OOD.

\paragraph{Why monolithic concatenated models underperform.}
A single score network trained on the full concatenated representation
$z=[z_1;\ldots;z_K]\!\in\!\mathbb{R}^D$, $D\!=\!\sum_k d_k$,
must enforce block-diagonal coupling in order to preserve the individual
geometry of each encoder's representation space: cross-encoder weight
sharing conflates the distinct distributional structures of different
encoders.
Under this constraint, the score-matching loss
$\mathbb{E}\bigl[\|s_\theta(z{+}\sigma\epsilon,t)+\epsilon/\sigma\|^2\bigr]$
is minimized by a score that acts independently on each partition---the
optimal score $s^*(z_k, t)$ in partition $k$ depends only on $z_k$---so
the model converges to a product of per-encoder marginals
$\log p_\theta(z) = \sum_k \log p_k(z_k)$,
equivalent to separate per-encoder models but at $2.3\times$ more parameters.

\paragraph{Monolithic baseline.}
For comparison we train a monolithic score network on the full concatenated
representation $z\!\in\!\mathbb{R}^{3328}$.
It uses a block-diagonal MLP with three encoder-aligned blocks of hidden
dimension $H_k\!=\!4d_k$ (yielding block widths $2048$, $3072$, $8192$ for
CLIP, DINOv2, and ResNet-50 respectively), a total hidden dimension of
$13{,}312$, and $L\!=\!6$ residual layers---trained with the same VP-SDE
objective as the per-encoder models.
As shown in Table~\ref{tab:main}, the $H_k\!=\!2d_k$ model fails on three
of four benchmarks (SVHN 0.799, CIFAR-100 0.617, CIFAR-10C 0.631); $H_k\!=\!4d_k$
is the minimum scale at which the monolithic model reaches $\geq\!0.90$ AUROC
across all shift types, ruling out the possibility that the cost gap is an
artefact of an over-parameterised baseline.

\paragraph{Model capacities.}
Table~\ref{tab:capacity} compares parameter counts.
The six per-encoder models total 285M parameters; the monolithic Score Network
baseline totals 647M.
Per-encoder training is also parallelizable and can target smaller GPU budgets.

\begin{table}[t]
  \caption{Parameter counts per model. $H_k = 2d_k$ for per-encoder models;
    the monolithic Score Network uses $H_k\!=\!4d_k$ per encoder block.}
  \label{tab:capacity}
  \centering
  \small
  \begin{tabular}{lrr}
    \toprule
    Model & $H$ & Params \\
    \midrule
    E1 (CLIP, $d_1\!=\!512$)      & 1024 &   7.5M \\
    E2 (DINOv2, $d_2\!=\!768$)    & 1536 &  16.7M \\
    E3 (ResNet, $d_3\!=\!2048$)   & 4096 & 118.0M \\
    \midrule
    Per-encoder suite (6 models)  & ---  & \textbf{285M} \\
    Monolithic Score Network      & 13312 & 647M \\
    \bottomrule
  \end{tabular}
\end{table}

% ============================================================
\section{Encoder Analysis and Combination Design}
\label{sec:analysis}
% ============================================================
\begin{table*}[t]
  \caption{ID-data diagnostics for all per-encoder models alongside observed Near-OOD
    performance.
    $\eta^2_\mathrm{c}$: ANOVA effect size on clean ID images.
    $\eta^2_\mathrm{corr}$: same, under synthetic corruptions.
    $\Delta\mu$: mean log-likelihood shift under synthetic corruptions. All from ID data only.}
  \label{tab:diagnostics}
  \centering
  \begin{tabular}{lcccccc}
    \toprule
    \textbf{Model} &
      $\boldsymbol{\eta^2_\mathrm{c}}$ &
      $\boldsymbol{\eta^2_\mathrm{corr}}$ &
      $\boldsymbol{\Delta\mu}$ &
      \textbf{SVHN} &
      \textbf{CIFAR-100} &
      \textbf{CIFAR-10C} \\
    \midrule
    E1 normed-fork  (CLIP)    & 0.096 & 0.092 &   282 & 0.977 & 0.841 & 0.907 \\
    E1 unnormed-fork  (CLIP)          & 0.133 & 0.106 &   180 & 0.975 & 0.830 & 0.887 \\
    E2 normed-fork  (DINOv2)  & 0.243 & 0.112 &    81 & 0.859 & 0.963 & 0.568 \\
    E2 unnormed-fork  (DINOv2)          & 0.154 & 0.084 &    93 & 0.868 & 0.964 & 0.569 \\
    E3 normed-fork  (ResNet)  & 0.109 & 0.046 & 4171  & 0.762 & 0.770 & 0.981 \\
    E3 unnormed-fork  (ResNet)         & 0.071 & 0.015 &  691  & 0.737 & 0.697 & 0.926 \\
    \bottomrule
  \end{tabular}
\end{table*}

\subsection{ID-Data Diagnostics: Characterising Test Power}
\label{sec:diagnostics}

Different near-OOD benchmarks present distinct pain points---global domain shift, fine-grained semantic shift, and
covariate corruption--- where prior methods most often
fail, each demanding a different kind of encoder sensitivity.
The following diagnostics surface these differences from ID data alone.

\paragraph{$\eta^2$: class-conditional F-test.}
Group $\{\ell_{k,f}(z_i)\}$ on the ID validation set by class label.
The ANOVA effect size
\begin{equation}
  \eta^2_{k,f} = \frac{\mathrm{SS}_\mathrm{between}}{\mathrm{SS}_\mathrm{total}}
  \label{eq:eta2}
\end{equation}
measures the fraction of log-likelihood variance explained by class identity.
High $\eta^2$ means the density model reflects fine-grained class structure,
predicting high sensitivity to semantic OOD (e.g.\ CIFAR-100).
Conversely, a low-$\eta^2$ model has learned a broad ID density that is not
overfit to class detail, predicting sensitivity to global domain shift (SVHN)
where class distinctions are irrelevant.
We find Spearman $\rho\!=\!0.771$ between $\eta^2$ and CIFAR-100 AUROC across all
six per-encoder models.

\paragraph{$\Delta\mu$: synthetic corruption probe.}
Apply a fixed synthetic corruption to clean ID test images
(GaussianBlur, kernel 3, $\sigma\!=\!1.0$; additive noise $\sigma\!=\!0.1$)
and compute:
\begin{equation}
  \Delta\mu_{k,f} = \mu_\mathrm{clean} - \mu_\mathrm{corrupt}.
  \label{eq:deltamu}
\end{equation}
All corrupted images are derived from ID data; no OOD dataset is consulted.
Large $\Delta\mu$ indicates corruption-sensitive features, predicting high power
against covariate shift.
We find Spearman $\rho=1.0$ between $\Delta\mu$ and CIFAR-10C AUROC.
DINOv2 is designed for corruption invariance~\citep{oquab2023dinov2}: its
$\Delta\mu\!\approx\!87$ (vs.\ 4171 for ResNet) correctly predicts its
CIFAR-10C AUROC of 0.568 vs.\ 0.981.

\paragraph{$\rho$: cross-encoder correlation screening.}
Spearman rank correlations between per-encoder log-likelihood vectors on
the ID test set quantify redundancy.
Within-encoder fork pairs (normed/unnormed) yield $\rho\!\approx\!0.92$:
nearly redundant on ID data, but differing in OOD-space geometry
(the normed fork amplifies corruption sensitivity; the unnormed fork
preserves raw semantic geometry).
Cross-encoder pairs yield $\rho\!\in\![0.17,0.39]$: genuinely complementary,
justifying multi-encoder combination.

\begin{remark}
\textit{The $\eta^2$ and $\Delta\mu$ diagnostics reveal a strict encoder specialisation:
no single encoder dominates across all shift types simultaneously
(Figure~\ref{fig:diagnostics}, Table~\ref{tab:diagnostics}).
CLIP maximises global domain-shift detection but underperforms on covariate shift;
DINOv2 leads on semantic OOD but is nearly blind to corruptions ($\Delta\mu\!\approx\!87$
vs.\ 4171 for ResNet);
ResNet-50 dominates CIFAR-10C but ranks last on SVHN and CIFAR-100.}
\end{remark}

Table~\ref{tab:diagnostics} further adds the $\eta^2$ under synthetic corruption
($\eta^2_\mathrm{corr}$).
Under corruptions, DINOv2's $\eta^2$ halves (0.243$\to$0.112): class structure
dissolves because DINOv2 features are explicitly trained to be corruption-invariant.
CLIP's $\eta^2$ is essentially stable (0.096$\to$0.092), consistent with its
known sensitivity to class-discriminative corruption patterns.
ResNet's $\eta^2$ drops sharply (0.109$\to$0.046) because low-level texture
features---the most corruption-sensitive---are class-agnostic.

\begin{figure}[t]
  \centering
  \includegraphics[width=\textwidth]{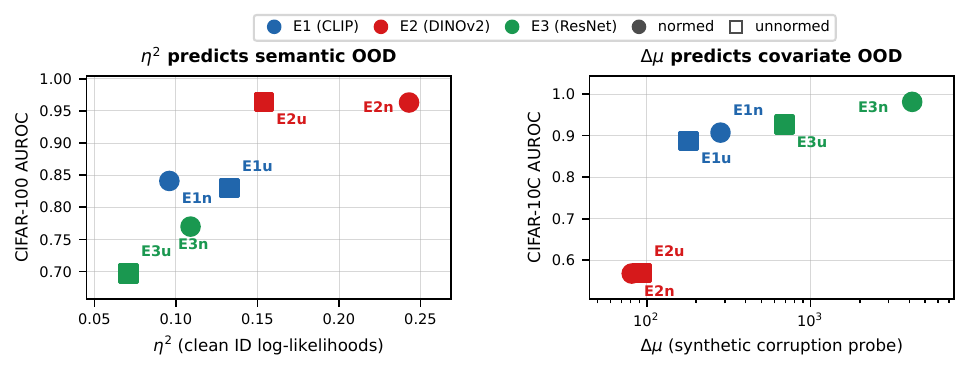}
  \caption{ID-data diagnostics predict OOD capability without observing any OOD
    samples.
    \textbf{Left}: $\eta^2$ (class structure in ID log-likelihoods) vs.\
    CIFAR-100 AUROC (Spearman $\rho\!=\!0.771$).
    \textbf{Right}: $\Delta\mu$ (score shift under synthetic corruptions of ID
    images) vs.\ CIFAR-10C AUROC (Spearman $\rho\!=\!1.0$).
    Points are labelled E1n/E1u/E2n/E2u/E3n/E3u (n~=~normed, u~=~unnormed).}
  \label{fig:diagnostics}
\end{figure}

\subsection{Two-Level $\min(\cdot)$-gate Combination}
\label{sec:combination}

Because the type of distribution shift is unknown at test time, any \emph{fixed}
weighting of encoder contributions---including a monolithic concatenated model---will
underperform on at least one shift type.
The combination must therefore be \emph{input-driven}: deferring to whichever
encoder raises the strongest alarm for a given test point, without committing to
any encoder in advance.

\begin{algorithm}[t]
\caption{EncMin2L: Two-Level Min-Gate OOD Detection}
\label{alg:encmin2l}
\begin{algorithmic}[1]
\Require Encoders $\{E_k\}_{k=1}^{K}$, each with normed and unnormed forks;
  ID validation set $\mathcal{Z}_\mathrm{val}$; FPR target $\alpha$
\Ensure OOD/ID decision for test point $x$
\Statex
\Statex \textbf{Calibration} \textit{(offline, once)}
\For{each encoder $k$, fork $f \in \{\mathrm{n},\mathrm{u}\}$}
  \State compute $r_{k,f}(z_i)$ for all $z_i \in \mathcal{Z}_\mathrm{val}$
    \Comment{PF-ODE log-likelihood}
\EndFor
\For{each encoder $k$}
  \State $e_k(z_i) \leftarrow \min\!\bigl(r_{k,\mathrm{n}}(z_i),\; r_{k,\mathrm{u}}(z_i)\bigr)$
    \Comment{Level~1: within-encoder min}
  \State fit empirical CDF $\hat{G}_k$ from $\{e_k(z_i)\}_{i}$
    \Comment{re-maps $\mathrm{Beta}(1,2) \to$ Uniform; Prop.~\ref{prop:beta}}
  \State $\hat{e}_k(z_i) \leftarrow \hat{G}_k\!\bigl(e_k(z_i)\bigr)$
\EndFor
\State $s(z_i) \leftarrow \min_{k}\;\hat{e}_k(z_i)$ for all $z_i$
  \Comment{Level~2: cross-encoder min}
\State $\tau \leftarrow \alpha$-quantile of $\{s(z_i)\}_{i}$
\Statex
\Statex \textbf{Inference} \textit{(online)}
\For{each encoder $k$, fork $f \in \{\mathrm{n},\mathrm{u}\}$}
  \State compute $r_{k,f}(x)$
    \Comment{PF-ODE log-likelihood}
\EndFor
\For{each encoder $k$}
  \State $e_k(x) \leftarrow \min\!\bigl(r_{k,\mathrm{n}}(x),\; r_{k,\mathrm{u}}(x)\bigr)$
  \State $\hat{e}_k(x) \leftarrow \hat{G}_k\!\bigl(e_k(x)\bigr)$
\EndFor
\State $s(x) \leftarrow \min_{k}\;\hat{e}_k(x)$
\State \Return \textbf{OOD} if $s(x) < \tau$, else \textbf{ID}
\end{algorithmic}
\end{algorithm}

Algorithm~\ref{alg:encmin2l} implements Tippett's minimum $p$-value
combination~\citep{westfall1993resampling}: declare OOD if any individual
test rejects.
By Proposition~\ref{prop:beta} below, $\hat{e}_k$ is approximately uniform
on ID data, so $s(z)$ is well-calibrated for threshold setting.

\begin{figure}[t]
  \centering
  % fig_arch.tex — encmin2L two-level min-gate architecture diagram.
% Input this file inside a \begin{figure}...\end{figure} environment.
% Requires: \usetikzlibrary{arrows.meta} in preamble.
\resizebox{0.75\columnwidth}{!}{%
\begin{tikzpicture}[
  >=Stealth,
  font=\scriptsize,
  % ---- node styles ----
  imgbox/.style={draw, rounded corners=3pt, fill=blue!6,
                 minimum width=1.7cm, minimum height=0.55cm, align=center},
  encbox/.style={draw, rounded corners=2pt, fill=gray!22,
                 minimum width=1.35cm, minimum height=0.52cm, align=center},
  ldmn/.style={draw, rounded corners=2pt, fill=blue!20,
               minimum width=1.05cm, minimum height=0.52cm, align=center},
  ldmu/.style={draw, rounded corners=2pt, fill=orange!30,
               minimum width=1.05cm, minimum height=0.52cm, align=center},
  gate1/.style={draw, rounded corners=2pt, fill=yellow!35, draw=yellow!60!black,
                minimum width=1.3cm, minimum height=0.5cm, align=center},
  cdfbox/.style={draw, rounded corners=2pt, fill=green!18, draw=green!55!black,
                 minimum width=1.3cm, minimum height=0.5cm, align=center},
  gate2/.style={draw, rounded corners=3pt, fill=yellow!40, draw=yellow!60!black,
                minimum width=3.8cm, minimum height=0.55cm, align=center},
  threshbox/.style={draw, rounded corners=3pt, fill=gray!15,
                    minimum width=2.0cm, minimum height=0.55cm, align=center},
  oodbox/.style={draw, rounded corners=3pt, fill=red!28, draw=red!60!black,
                 minimum width=1.2cm, minimum height=0.55cm, align=center,
                 font=\bfseries\scriptsize},
  idbox/.style={draw, rounded corners=3pt, fill=green!28, draw=green!60!black,
                minimum width=1.2cm, minimum height=0.55cm, align=center,
                font=\bfseries\scriptsize},
  lbl/.style={draw=none, fill=none, inner sep=1pt, font=\scriptsize},
  arr/.style={->, thick},
]

% ===================== ROW 0: INPUT =====================
\node[imgbox] (img) at (0, 0) {Input $x$};

% ===================== ROW 1: ENCODERS =====================
\node[encbox] (E1) at (-2.5, -1.3) {E1 (CLIP)};
\node[encbox] (E2) at ( 0.0, -1.3) {E2 (DINOv2)};
\node[encbox] (E3) at ( 2.5, -1.3) {E3 (ResNet)};

\draw[arr] (img.south) -- (E1.north);
\draw[arr] (img.south) -- (E2.north);
\draw[arr] (img.south) -- (E3.north);

% ===================== ROW 2: RDMs (normed=blue, unnormed=orange) =====================
\node[ldmn] (L1n) at (-3.2, -2.65) {RDM\\[-1pt]normed};
\node[ldmu] (L1u) at (-1.8, -2.65) {RDM\\[-1pt]unnorm.};
\node[ldmn] (L2n) at (-0.7, -2.65) {RDM\\[-1pt]normed};
\node[ldmu] (L2u) at ( 0.7, -2.65) {RDM\\[-1pt]unnorm.};
\node[ldmn] (L3n) at ( 1.8, -2.65) {RDM\\[-1pt]normed};
\node[ldmu] (L3u) at ( 3.2, -2.65) {RDM\\[-1pt]unnorm.};

\draw[arr] (E1.south) -- (L1n.north);
\draw[arr] (E1.south) -- (L1u.north);
\draw[arr] (E2.south) -- (L2n.north);
\draw[arr] (E2.south) -- (L2u.north);
\draw[arr] (E3.south) -- (L3n.north);
\draw[arr] (E3.south) -- (L3u.north);

% p-value annotations (below each RDM); unnormed labels shifted right to clear arrow
\node[lbl] at (-3.2, -3.3) {$r_{1,\mathrm{n}}$};
\node[lbl] at (-1.55, -3.3) {$r_{1,\mathrm{u}}$};
\node[lbl] at (-0.7, -3.3) {$r_{2,\mathrm{n}}$};
\node[lbl] at ( 0.95, -3.3) {$r_{2,\mathrm{u}}$};
\node[lbl] at ( 1.8, -3.3) {$r_{3,\mathrm{n}}$};
\node[lbl] at ( 3.45, -3.3) {$r_{3,\mathrm{u}}$};

% ===================== ROW 3: LEVEL 1 MIN GATES =====================
\node[gate1] (G1) at (-2.5, -4.1) {$\min(\cdot{,}\cdot)$};
\node[gate1] (G2) at ( 0.0, -4.1) {$\min(\cdot{,}\cdot)$};
\node[gate1] (G3) at ( 2.5, -4.1) {$\min(\cdot{,}\cdot)$};

\draw[arr] (L1n.south) -- (G1.north);
\draw[arr] (L1u.south) -- (G1.north);
\draw[arr] (L2n.south) -- (G2.north);
\draw[arr] (L2u.south) -- (G2.north);
\draw[arr] (L3n.south) -- (G3.north);
\draw[arr] (L3u.south) -- (G3.north);

% "Level 1" labels in the gaps between the three min gates
\node[lbl, font=\scriptsize] at (-1.25, -4.1) {Level~1};
\node[lbl, font=\scriptsize] at ( 1.25, -4.1) {Level~1};

% e_k annotations (below each gate); shifted right to clear the arrow
\node[lbl] at (-2.25, -4.65) {$e_1$};
\node[lbl] at ( 0.25, -4.65) {$e_2$};
\node[lbl] at ( 2.75, -4.65) {$e_3$};

% ===================== ROW 4: RE-CDF (Prop. 1 recalibration) =====================
\node[cdfbox] (C1) at (-2.5, -5.3) {$\hat{G}_1(e_1)$};
\node[cdfbox] (C2) at ( 0.0, -5.3) {$\hat{G}_2(e_2)$};
\node[cdfbox] (C3) at ( 2.5, -5.3) {$\hat{G}_3(e_3)$};

\draw[arr] (G1.south) -- (C1.north);
\draw[arr] (G2.south) -- (C2.north);
\draw[arr] (G3.south) -- (C3.north);

% ê_k annotations (below each CDF box)
\node[lbl] at (-2.5, -5.85) {$\hat{e}_1$};
\node[lbl] at ( 0.25, -5.85) {$\hat{e}_2$};
\node[lbl] at ( 2.5, -5.85) {$\hat{e}_3$};

% ===================== ROW 5: LEVEL 2 MIN GATE =====================
\node[gate2] (G2L) at (0, -6.55)
  {Level~2 (cross-encoder): $\displaystyle\min_k \hat{e}_k = s(z)$};

\draw[arr] (C1.south) -- (G2L.north);
\draw[arr] (C2.south) -- (G2L.north);
\draw[arr] (C3.south) -- (G2L.north);

% ===================== ROW 6: THRESHOLD =====================
\node[threshbox] (thresh) at (0, -7.45) {$s(z) < \tau\;$?};
\draw[arr] (G2L.south) -- (thresh.north);

% ===================== ROW 7: DECISION =====================
\node[oodbox] (ood) at (-1.9, -8.4) {\textbf{OOD}};
\node[idbox]  (id)  at ( 1.9, -8.4) {\textbf{ID}};

\draw[arr] (thresh.south) to[out=-130, in=90]
  node[lbl, left, yshift=7pt] {yes} (ood.north);
\draw[arr] (thresh.south) to[out=-50, in=90]
  node[lbl, right, yshift=7pt] {no}  (id.north);

\end{tikzpicture}%
}% end \resizebox
  \caption{Architecture of the two-level $\min(\cdot)$-gate (EncMin2L).
    Each encoder produces two $p$-values: $r_{k,\mathrm{n}}$ (blue, normed repr.)
    and $r_{k,\mathrm{u}}$ (orange, unnormed repr.).
    Level~1 takes the within-encoder minimum, re-CDF recalibration
    (Proposition~\ref{prop:beta}) restores uniformity, and Level~2
    applies the cross-encoder Tippett minimum
    (Algorithm~\ref{alg:encmin2l}) to produce the final score $s(z)$.}
  \label{fig:arch}
\end{figure}
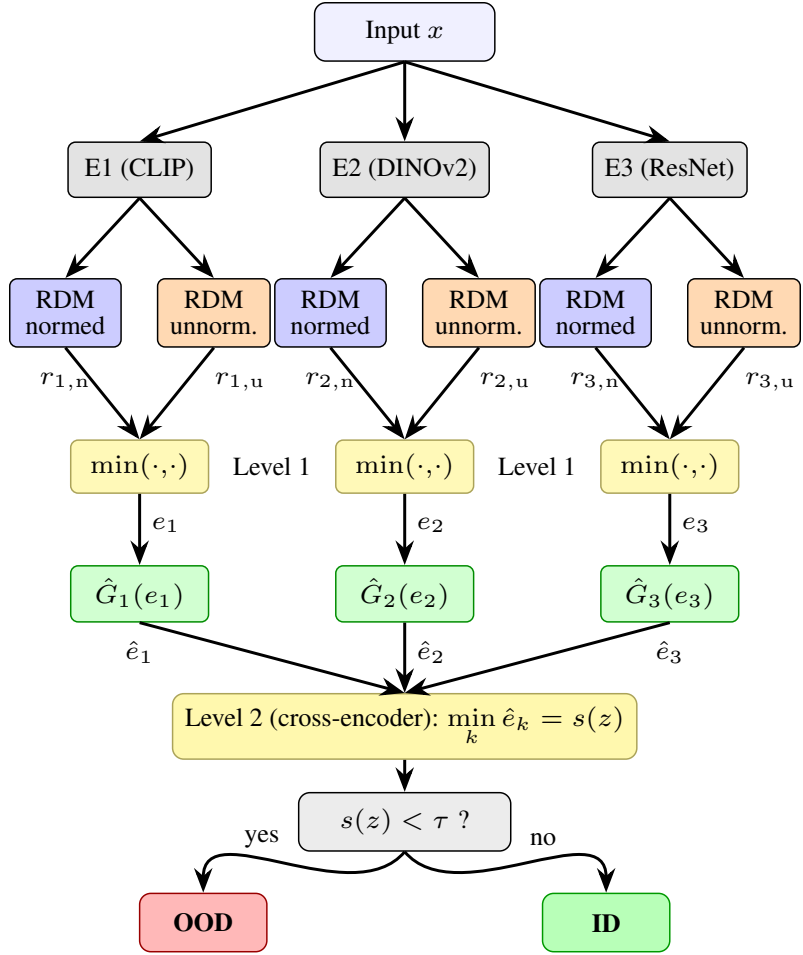

\subsection{Theoretical Justification}
\label{sec:theory}

\begin{proposition}[Re-calibration between levels]
  \label{prop:beta}
  Let $U_1, U_2 \overset{iid}{\sim} \mathrm{Uniform}(0,1)$.
  Then $M\!=\!\min(U_1,U_2)\sim\mathrm{Beta}(1,2)$ with
  $\mathbb{E}[M]\!=\!\tfrac{1}{3}$ and PDF $f_M(m)\!=\!2(1-m)$.
  More generally, $\min_k U_k \sim \mathrm{Beta}(1,K)$.
  Applying the empirical CDF $\hat{G}_k$ of $\{e_k(z_i)\}_{i}$ (ID validation)
  re-maps $e_k$ to an approximately uniform distribution, restoring the
  uniformity assumption for the Level~2 combination.
\end{proposition}

\begin{proposition}[Min-gate power inheritance]
  \label{prop:power}
  Let $U_1,\ldots,U_K$ be the CDF-transformed per-encoder scores produced by
  Algorithm~\ref{alg:encmin2l}, and let $M = \min_k U_k$.
  For any OOD distribution $q$ and threshold $\tau \in [0,1]$:
  \[
    P_{x \sim q}(M < \tau) \;\geq\; \max_{k}\; P_{x \sim q}(U_k < \tau).
  \]
  \begin{proof}
    $M \leq U_k$ for all $k$, so $\{M < \tau\} \supseteq \{U_k < \tau\}$
    for every $k$; taking probabilities gives the inequality.
  \end{proof}
\end{proposition}

\begin{remark}
  \textit{Proposition~\ref{prop:power} is unconditional: it requires no
  distributional assumptions on $q$.
  The min-gate inherits the detection power of whichever constituent encoder
  is most discriminative for the OOD type at hand---without prior knowledge of
  which that will be.
  The only assumption in the full system is encoder independence under
  $p_\mathrm{ID}$, needed for the Beta$(1,K)$ null in
  Proposition~\ref{prop:beta}; the power guarantee holds regardless.}
\end{remark}

% ============================================================
\section{Main Results and Experiments}
\label{sec:experiments}
% ============================================================

EncMin2L is our proposed two-level $\min(\cdot)$-gate combination (\S\ref{sec:combination}):
it aggregates per-encoder log-likelihood $p$-values via within-encoder minimum
(Level~1), re-CDF recalibration (Proposition~\ref{prop:beta}), and cross-encoder
Tippett minimum (Level~2), calibrated on ID validation data alone with no OOD
labels at any stage.
We evaluate EncMin2L against classifier-based, feature-space, and diffusion-based
baselines in the following experimental setting.

\subsection{Setup}
\label{sec:setup}

\textbf{ID datasets.}
We evaluate under two settings.
\emph{CIFAR-10}~\citep{cifar100}: $32\!\times\!32$ RGB, bicubic-upsampled to $224\!\times\!224$ for
encoding; 50K train, 10K test with 10\% held out for validation.
\emph{CIFAR-100}~\citep{cifar100}: same resolution, encoder preprocessing, and split
protocol, with 100 classes.

\textbf{OOD benchmarks.}
Under CIFAR-10 as ID:
SVHN~\cite{netzer2011reading} (global domain shift),
CIFAR-100~\cite{cifar100} (fine-grained semantic shift),
CIFAR-10C~\citep{hendrycks2019robustness} (covariate shift, 15 corruption types),
DTD~\cite{cimpoi14describing} (texture), and
CelebA~\cite{liu2015faceattributes} (face domain).
Under CIFAR-100 as ID:
CIFAR-10 (coarser-grained semantic shift), SVHN, DTD, and CelebA.
CIFAR-10C is omitted from the CIFAR-100 setting as no matched corruption benchmark
exists for CIFAR-100.

\textbf{Metrics.} AUROC ($\uparrow$) and FPR at 95\% TPR (FPR@95, $\downarrow$;
in parentheses throughout).
Higher AUROC and lower FPR indicate better detection.

\textbf{Baselines.}
\emph{Per-encoder}: each of the six single-encoder models (E1/E2/E3
$\times$ normed/unnormed) as individual detectors.
\emph{Monolithic}: Score Network trained on the full $D\!=\!3328$
concatenated representation in normed, unnormed, and dual-fork configurations.
\emph{Classifier-based}: Maximum Softmax Probability
(MSP)~\citep{hendrycks2017baseline}, Mahalanobis
distance~\citep{lee2018mahalanobis}, and energy score~\citep{liu2020energy}
applied to a ResNet-50 classifier trained on the respective ID dataset.
\emph{Feature-space}: $k$-NN distance~\citep{sun2022knn} ($k\!=\!50$) in the
representation space of each pretrained encoder.

\subsection{Main Results}

Table~\ref{tab:main} reports results for all evaluated methods under CIFAR-10 as ID.
EncMin2L achieves $\geq\!0.94$ AUROC on all four benchmarks simultaneously---a
bar no individual encoder or monolithic baseline clears---outperforming the best
monolithic competitor (dual-fork Score Network) by $+0.024$ on SVHN and
$+0.047$ on CIFAR-100 while using $2.3\!\times$ fewer parameters.

The encoder specialisation predicted by the \S\ref{sec:diagnostics} diagnostics
is directly visible in Table~\ref{tab:main}.
CLIP (E1) leads on global domain shift (SVHN 0.977) but underperforms where
semantic class structure matters, consistent with its low $\eta^2$.
DINOv2 (E2) leads on semantic OOD (CIFAR-100 0.964) yet collapses on covariate
shift (CIFAR-10C 0.568), exactly as its near-zero $\Delta\mu$ predicted.
ResNet-50 (E3) inverts this: its large $\Delta\mu$ flags it as the covariate
specialist, and it peaks at CIFAR-10C 0.981 while trailing on SVHN and
CIFAR-100.
Because each encoder has a structural blind spot on at least one shift type,
EncMin2L's $\min(\cdot)$-gate defers to whichever encoder raises the strongest alarm per
test point---and its worst-case AUROC of 0.951 (SVHN) reflects that no shift
type is left uncovered. DTD and CelebA are saturated across all our models (AUROC $\geq\!0.999$ on DTD; $1.000$ on CelebA with FPR $0.000$), reflecting the extreme perceptual distance of texture patches and face images from CIFAR-10 objects; we discuss both jointly in \S\ref{sec:comparison}.
The monolithic dual-fork score network is the strongest monolithic competitor,
yet EncMin2L surpasses it on both SVHN and CIFAR-100.

Under CIFAR-100 as ID (Table~\ref{tab:cifar100}), the diagnostics again correctly
predict encoder roles: E2 (DINOv2) leads on CIFAR-10 as OOD ($0.923$), benefiting
from its high $\eta^2$ when the ID manifold spans 100 classes and the shift is to
coarser-grained objects; E1 (CLIP) retains its SVHN advantage ($0.903$).
EncMin2L does not achieve $\geq\!0.94$ on SVHN ($0.806$): E2 and E3 both assign
high likelihood to SVHN under the richer 100-class ID distribution, creating a
per-encoder floor the $\min(\cdot)$-gate cannot overcome.
Monolithic models remain the weakest near-OOD detectors, extending the
isotropic noise finding beyond CIFAR-10.

\begin{table*}[t]
  \caption{\textbf{CIFAR-10 as ID}: AUROC (FPR@95\%TPR) on five OOD benchmarks.
    \textbf{Bold}: \textbf{EncMin2L is the only method with $\geq$0.94 on all four non-saturated benchmarks.}
    ``---'': not evaluated by that work.
    VDM uses a ResNet-18 encoder; DiffPath is a CelebA-trained pixel-space DM applied zero-shot.}
  \label{tab:main}
  \centering
  \scriptsize
  \begin{tabular}{lccccc}
    \toprule
    & \multicolumn{3}{c}{\textit{Near-OOD}} & \multicolumn{2}{c}{\textit{Far-OOD}} \\
    \cmidrule(lr){2-4}\cmidrule(lr){5-6}
    \textbf{Model} &
      \textbf{SVHN} &
      \textbf{CIFAR-100} &
      \textbf{CIFAR-10C} &
      \textbf{DTD} &
      \textbf{CelebA} \\
    \midrule
    \multicolumn{6}{l}{\textit{Single-encoder per-fork models}} \\
    E1 normed-fork \ (CLIP, 512-d)       & 0.977 (0.139) & 0.841 (0.479) & 0.907 (0.551) & 1.000 (0.000) & 1.000 (0.000) \\
    % E1 unnorm                     & 0.975 (0.160) & 0.830 (0.491) & 0.887 (0.612) & 1.000 (0.000) & 1.000 (0.000) \\
    % E2 norm \ (DINOv2, 768-d)     & 0.859 (0.483) & 0.963 (0.176) & 0.568 (0.965) & 1.000 (0.001) & 1.000 (0.003) \\
    E2 raw-fork \ (DINOv2, 768-d)        & 0.868 (0.410) & 0.964 (0.163) & 0.569 (0.957) & 1.000 (0.001) & 1.000 (0.003) \\
    E3 normed-fork \ (ResNet-50, 2048-d) & 0.762 (0.811) & 0.770 (0.613) & 0.981 (0.090) & 0.999 (0.004) & 1.000 (0.000) \\
    % E3 unnorm                     & 0.737 (0.847) & 0.697 (0.797) & 0.926 (0.381) & 0.998 (0.006) & 1.000 (0.000) \\
    \midrule
    \multicolumn{6}{l}{\textit{Monolithic multi-encoder models}} \\
    Score Network ($H_k\!=\!2d_k$) & 0.799 (0.859) & 0.617 (0.981) & 0.631 (0.949) & 0.961 (0.171) & 1.000 (0.000) \\
    Score Network ($H_k\!=\!4d_k$) & 0.927 (0.388) & 0.895 (0.388) & 0.973 (0.127) & 1.000 (0.000) & 0.999 (0.000) \\
    \midrule
    \textbf{EncMin2L (ours)} & \textbf{0.951 (0.309)} & \textbf{0.942 (0.255)} & \textbf{0.963 (0.195)} & \textbf{0.999 (0.000)} & \textbf{1.000 (0.000)} \\
    \midrule
    \multicolumn{6}{l}{\textit{Classifier-based methods (ResNet-50 on CIFAR-10)}} \\
    MSP~\citep{hendrycks2017baseline}           & 0.916 (0.544) & 0.885 (0.599) & 0.915 (0.534) & 0.899 (0.548) & 0.856 (0.705) \\
    Mahalanobis~\citep{lee2018mahalanobis}      & 0.957 (0.238) & 0.830 (0.676) & 0.981 (0.093) & 0.957 (0.214) & 0.772 (0.779) \\
    Energy~\citep{liu2020energy}                & 0.916 (0.409) & 0.862 (0.533) & 0.885 (0.583) & 0.849 (0.562) & 0.809 (0.708) \\
    \midrule
    \multicolumn{6}{l}{\textit{Feature-space methods (same pretrained encoders, $k$=50)}} \\
    $k$-NN~\citep{sun2022knn} (CLIP repr.)      & 0.930 (0.564) & 0.911 (0.439) & 0.921 (0.612) & 1.000 (0.000) & 1.000 (0.000) \\
    $k$-NN~\citep{sun2022knn} (DINOv2 repr.)    & 0.978 (0.157) & 0.976 (0.110) & 0.867 (0.746) & 1.000 (0.001) & 1.000 (0.001) \\
    $k$-NN~\citep{sun2022knn} (ResNet repr.)    & 0.970 (0.168) & 0.775 (0.784) & 0.936 (0.395) & 0.990 (0.051) & 0.911 (0.668) \\
    \midrule
    \multicolumn{6}{l}{\textit{Recent diffusion-based OOD methods}} \\
    EigenScore~\citep{shoushtari2026eigenscore}    & 0.810 (---) & 0.880 (---) & ---           & 0.756 (---) & 0.873 (---) \\
    VDM log $p_T$~\citep{jarve2025vdm}            & 0.939 (0.213) & 0.885 (0.441) & ---         & 0.946 (0.227) & --- \\
    DiffPath-6D~\citep{heng2024diffpath}           & 0.910 (---) & 0.590 (---) & ---           & 0.923 (---) & 0.897 (---) \\
    DM Dual Threshold~\citep{kamkari2024geometric} & 0.944 (---) & 0.560 (---) & ---           & ---         & 0.648 (---) \\
    LMD~\citep{liu2023lmd}                         & 0.992 (---) & 0.607 (---) & ---           & ---         & 0.834 (---) \\
    Proj.\ Regret~\citep{choi2023projection}       & 0.993 (---) & 0.775 (---) & ---           & 0.910 (---) & --- \\
    \bottomrule
  \end{tabular}
\end{table*}

\subsection{Comparison with Feature-Space and Prior Methods}
\label{sec:comparison}

\begin{table*}[t]
  \caption{\textbf{CIFAR-100 as ID}: AUROC (FPR@95\%TPR) on four OOD benchmarks.}
  \label{tab:cifar100}
  \centering
  \scriptsize
  \begin{tabular}{lcccc}
    \toprule
    & \multicolumn{2}{c}{\textit{Near-OOD}} & \multicolumn{2}{c}{\textit{Far-OOD}} \\
    \cmidrule(lr){2-3}\cmidrule(lr){4-5}
    \textbf{Model} &
      \textbf{CIFAR-10} &
      \textbf{SVHN} &
      \textbf{DTD} &
      \textbf{CelebA} \\
    \midrule
    \multicolumn{5}{l}{\textit{Single-encoder per-fork models}} \\
    E1 normed-fork \ (CLIP, 512-d)       & 0.706 (0.800) & 0.903 (0.502) & 1.000 (0.000) & 1.000 (0.000) \\
    E2 normed-fork \ (DINOv2, 768-d)     & 0.923 (0.333) & 0.642 (0.750) & 0.995 (0.017) & 0.947 (0.336) \\
    E3 normed-fork \ (ResNet-50, 2048-d) & 0.627 (0.885) & 0.519 (0.962) & 0.994 (0.021) & 0.983 (0.082) \\
    \midrule
    \multicolumn{5}{l}{\textit{Monolithic multi-encoder model}} \\
    Score Network ($H_k\!=\!4d_k$) & 0.746 (0.796) & 0.744 (0.758) & 0.999 (0.006) & 0.998 (0.006) \\
    \midrule
    \textbf{EncMin2L (ours)} & \textbf{0.892 (0.428)} & \textbf{0.806 (0.648)} & \textbf{1.000 (0.000)} & \textbf{1.000 (0.000)} \\
    \midrule
    \multicolumn{5}{l}{\textit{Classifier-based methods (ResNet-50 on CIFAR-100)}} \\
    MSP~\citep{hendrycks2017baseline}           & 0.741 (0.843) & 0.784 (0.786) & 0.719 (0.874) & 0.632 (0.979) \\
    Mahalanobis~\citep{lee2018mahalanobis}      & 0.483 (0.981) & 0.884 (0.475) & 0.897 (0.359) & 0.282 (0.998) \\
    Energy~\citep{liu2020energy}                & 0.754 (0.831) & 0.852 (0.741) & 0.764 (0.869) & 0.451 (0.988) \\
    \midrule
    \multicolumn{5}{l}{\textit{Feature-space methods (same pretrained encoders, $k$=50)}} \\
    $k$-NN~\citep{sun2022knn} (CLIP repr.)      & 0.699 (0.786) & 0.734 (0.818) & 0.999 (0.003) & 1.000 (0.000) \\
    $k$-NN~\citep{sun2022knn} (DINOv2 repr.)    & 0.930 (0.352) & 0.881 (0.665) & 0.992 (0.030) & 0.909 (0.560) \\
    $k$-NN~\citep{sun2022knn} (ResNet repr.)    & 0.642 (0.869) & 0.770 (0.892) & 1.000 (0.002) & 0.898 (0.440) \\
    \midrule
    \multicolumn{5}{l}{\textit{Recent diffusion-based OOD methods}} \\
    EigenScore~\citep{shoushtari2026eigenscore}    & 0.642 (---) & 0.661 (---) & 0.627 (---) & 0.427 (---) \\
    VDM log $p_T$~\citep{jarve2025vdm}            & 0.777 (0.605) & 0.755 (0.633) & 0.758 (0.642) & --- \\
    DiffPath-6D~\citep{heng2024diffpath}           & 0.483 (---) & 0.724 (---) & 0.761 (---) & 0.887 (---) \\
    DM Dual Threshold~\citep{kamkari2024geometric} & 0.791 (---) & 0.945 (---) & --- & 0.902 (---) \\
    LMD~\citep{liu2023lmd}                         & 0.568 (---) & 0.985 (---) & --- & 0.595 (---) \\
    Proj.\ Regret~\citep{choi2023projection}       & 0.577 (---) & 0.945 (---) & 0.91 (---) & --- \\
    \bottomrule
  \end{tabular}
\end{table*}

\textbf{Classifier-based methods} use a supervised ResNet-50 trained on CIFAR-10.
Mahalanobis distance matches or exceeds EncMin2L on SVHN (0.957) and
CIFAR-10C (0.981): class-conditional Gaussian distances capture large-scale domain
and covariate shift well when ID class structure is informative of the shift.
On CIFAR-100, however, Mahalanobis collapses to 0.830---the 10-class covariance
model is too coarse to separate 100 semantically adjacent categories, exactly the
limitation the $\eta^2$ framing exposes: a covariance model can only be as
fine-grained as the label set it was trained with.
MSP and energy score are weaker overall.
None of the three achieves $\geq\!0.94$ simultaneously across all four benchmarks;
EncMin2L closes this gap without using class labels at any stage.

Under CIFAR-100, classifier-based collapse on CelebA is far more severe:
Mahalanobis drops to $0.282$ (from $0.772$ on CIFAR-10) and Energy to $0.451$.
With 100 class-conditional Gaussians the covariance model becomes too fragmented to
maintain a coherent face-vs-object boundary.
On CIFAR-10 as OOD, Mahalanobis also collapses to $0.483$: classes shared between
CIFAR-10 and CIFAR-100 (vehicles, animals) lie geometrically near the ID means,
making Mahalanobis distance unreliable precisely where ID and near-OOD share
semantic overlap.

\textbf{Feature-space $k$-NN} ($k$=50) on the same pretrained features uses identical
representations but replaces the learned density with a non-parametric distance score.
Under CIFAR-10 as ID, DINOv2 $k$-NN (E2) exceeds EncMin2L on SVHN (0.978 vs.\ 0.951) and CIFAR-100
(0.976 vs.\ 0.942), yet collapses on CIFAR-10C (0.867).
The diagnostics explain this directly: DINOv2's corruption-invariant features
($\Delta\mu\!=\!28$) render training neighbours nearly indistinguishable from
corrupted images, so nearest-neighbour proximity is not a reliable covariate-shift signal.
ResNet-50 $k$-NN inverts the pattern---strong on SVHN (0.970) and CIFAR-10C (0.936)
but poor on CIFAR-100 (0.775)---its large $\Delta\mu$ conferring covariate sensitivity
at the cost of fine-grained semantic discrimination.
No single-encoder $k$-NN achieves $\geq\!0.94$ on all four benchmarks.

Under CIFAR-100, DINOv2 $k$-NN (E2) becomes the strongest single method on
near-OOD benchmarks (CIFAR-10: $0.930$; SVHN: $0.881$).
The same encoder that collapsed on covariate shift under CIFAR-10 leads on
semantic near-OOD under CIFAR-100, confirming that encoder utility is
shift-type-relative: E2's high $\eta^2$ is an asset when the ID space is
semantically richer and the shift is to coarser-grained classes.

\textbf{DTD and CelebA: saturation gap.}
All per-encoder models and EncMin2L achieve AUROC $\geq\!0.999$ on DTD and $1.000$ on CelebA
(FPR $0.000$), as do CLIP and DINOv2 $k$-NN.
Classifier-based methods do not: MSP $0.856$, Energy $0.809$, Mahalanobis $0.772$---a
supervised ResNet-50 trained on CIFAR-10 object classes does not reliably flag face images
because they have no equivalent class boundary to cross.
Pixel-space diffusion methods similarly fail to saturate: EigenScore $0.873$, DiffPath
$0.897$, LMD $0.834$, DM Dual Threshold $0.648$ on CelebA; EigenScore $0.756$ and
Projection Regret $0.910$ on DTD.
Modelling the ID distribution in representation space, where face and texture images
have no support, is what makes the gap disappear.

Under CIFAR-100, DTD remains fully saturated ($\geq\!0.992$) for all our models.
CelebA shows minor degradation for E2 ($0.947$) and E3 ($0.983$) normed
forks---the 100-class ID manifold marginally reduces the support gap for face
images.
Classifier degradation on CelebA intensifies further: Mahalanobis reaches $0.282$,
reinforcing that representation-space density modelling is insensitive to ID label
granularity in a way class-boundary distances are not.

\textbf{Head-to-head: Representation-space diffusion.}
\citet{jarve2025vdm} independently apply the same core hypothesis on the same
CIFAR-10 ID dataset, training a Variational Diffusion Model (VDM) on ResNet-18
penultimate features (512-d).
EncMin2L outperforms their best likelihood-based score (VDM log $p_T$) by
$+0.012$ on SVHN (0.951 vs.\ 0.939) and $+0.057$ on CIFAR-100 (0.942 vs.\ 0.885),
consistent with encoder diversity:
CLIP and DINOv2 capture semantic cues inaccessible to a single ResNet-18. They also report TKDL---a class-conditioned score requiring labels at
inference---as their best single-encoder method, reaching 0.900 on SVHN.
EncMin2L surpasses TKDL by $+0.051$ while remaining fully unsupervised, showing
that multi-encoder combination recovers gains otherwise achievable only through supervision.

Under CIFAR-100 as ID, VDM degrades to $0.777$ on CIFAR-10 as OOD and $0.755$ on SVHN;
EncMin2L maintains its advantage ($+0.115$, $+0.051$).
A single ResNet-18 encoder lacks the semantic capacity to model a 100-class ID manifold,
and the multi-encoder gap widens as ID complexity grows.

\begin{remark}
\textit{\citet{ding2025rdm} independently apply a sub-VP SDE score model on a single
pretrained encoder's representation space at ImageNet scale, validating the
same core hypothesis.
Architecturally, this is equivalent to the single-encoder models in
Table~\ref{tab:main}---a setting where EncMin2L consistently outperforms
every individual encoder across all four benchmarks.}
\end{remark}

\begin{remark}
\textit{Representation-space diffusion methods are architecturally closest to
EncMin2L---the same PF-ODE log-likelihood, the same compressed representation space.
Under CIFAR-10 ID, the gap is already explained by single-encoder scope.
Under CIFAR-100, VDM's steeper degradation confirms that multi-encoder coverage
is the variable that scales with ID complexity, not the choice of diffusion formulation.}
\end{remark}

\textbf{Head-to-head: Pixel-space diffusion.}
\citet{heng2024diffpath} propose DiffPath, which extracts a 6-dimensional statistic
from the score-function trajectory of a pixel-space DM and fits a GMM in that space.
With CIFAR-10 as ID, a CelebA-trained model applied zero-shot yields AUROC 0.910
on SVHN, 0.590 on CIFAR-100, and 0.923 on DTD.
EncMin2L outperforms on both: $+0.041$ on SVHN and $+0.352$ on CIFAR-100.
The CIFAR-100 gap is particularly large: pixel-space trajectories capture global
distribution shift but miss the fine-grained semantic structure that representation-space
log-likelihoods encode---precisely the class-conditional variance that $\eta^2$
($0.096$--$0.243$) measures and that is absent in raw pixel trajectories.
DiffPath's zero-shot universality comes at the cost of ID-specific semantic sensitivity;
EncMin2L trades universality for precise density modelling on the ID manifold.

Under CIFAR-100 as ID, DiffPath's near-OOD performance collapses to $0.483$ on
CIFAR-10 as OOD---near chance---while SVHN recovers to $0.724$.
CelebA remains at $0.887$: face images present a coarse perceptual contrast that is
relatively ID-distribution-agnostic.
The pixel-space trajectory cannot resolve the coarser-grained/finer-grained class
boundary that separates CIFAR-10 from a CIFAR-100 ID distribution.

\textbf{Head-to-head: LL+LID dual-threshold diffusion.}
\citet{kamkari2024geometric} flag OOD if log-likelihood $\log p_\theta(\mathbf{x})$
or local intrinsic dimension (LID) falls below a threshold, targeting the paradox that
low-dimensional OOD data can receive high density yet low probability mass.
Under CIFAR-10 as ID, their best DM variant achieves AUROC~0.944 on SVHN and 0.560 on CIFAR-100;
EncMin2L exceeds both by $+0.007$ and $+0.382$ respectively.
The large CIFAR-100 gap---despite the principled LL+LID combination---confirms that
pixel-space LID cannot substitute for the semantic class structure that
representation-space log-likelihoods encode.

Under CIFAR-100 as ID, DM Dual Threshold's SVHN AUROC rises to $0.945$, now exceeding
EncMin2L ($0.806$): the richer ID distribution sharpens the global domain gap for
pixel-space LL.
On CIFAR-10 as OOD it reaches only $0.791$ ($-0.101$ vs.\ EncMin2L's $0.892$):
near-OOD failure deepens as ID and OOD share an increasingly overlapping class taxonomy.

\textbf{Head-to-head: Inpainting-based diffusion.}
\citet{liu2023lmd} propose LMD, which masks an image with a checkerboard pattern,
reconstructs via DDPM inpainting, and scores OOD-ness by LPIPS perceptual distance.
Under CIFAR-10 as ID, LMD reaches 0.992 on SVHN, surpassing EncMin2L (0.951): inpainting fails visibly
on house-number images, making the perceptual gap large and discriminative.
On CIFAR-100, LMD drops to 0.607 ($-0.335$ vs.\ EncMin2L's 0.942): CIFAR-100 shares
visual structure with CIFAR-10, so DDPM reconstruction succeeds and the perceptual
score collapses---the class-conditional density structure ($\eta^2 = 0.096$--$0.243$)
that drives our detection is invisible to pixel reconstruction.

Under CIFAR-100 as ID, the asymmetry intensifies: LMD reaches $0.985$ on SVHN
(exceeding EncMin2L's $0.806$) but collapses to $0.568$ on CIFAR-10 as OOD.
Reconstruction succeeds precisely when visual structure is shared between ID and OOD,
and CIFAR-10's semantic overlap with CIFAR-100 is near-maximal.

\textbf{Head-to-head: Projection-based diffusion.}
\citet{choi2023projection} propose Projection Regret (PR), scoring OOD-ness by the
LPIPS distance between a test image and its diffusion-based projection, with recursive
projections to cancel background statistics.
Under CIFAR-10 as ID, PR achieves AUROC~0.993 on SVHN, exceeding EncMin2L by $+0.042$: perceptual distance
is highly discriminative when global domain statistics differ sharply.
On CIFAR-100, however, PR drops to 0.775 ($-0.167$ vs.\ EncMin2L's 0.942): CIFAR-100
shares background statistics with CIFAR-10, so the projection reconstructs OOD images
convincingly and the score collapses precisely where $\eta^2$ is highest.
EncMin2L sacrifices only $-0.042$ on SVHN for AUROC$\geq\!0.94$ across all benchmarks--- balancing semantic sensitivity for global-shift detection.

Under CIFAR-100 as ID, Projection Regret reaches $0.945$ on SVHN (vs.\ EncMin2L's
$0.806$) but drops to $0.577$ on CIFAR-10 as OOD and $0.910$ on DTD
(vs.\ our $1.000$).
Background-statistics cancellation is effective when the global domain gap is large;
it provides no signal when ID and OOD share the same object taxonomy.

\textbf{Head-to-head: Posterior-covariance diffusion.}
\citet{shoushtari2026eigenscore} propose EigenScore, which uses the
eigenvalue spectrum of a diffusion denoiser's posterior covariance as the OOD signal:
OOD inputs inflate the leading eigenvalues, yielding a Jacobian-free spectral detector.
Under CIFAR-10 as ID, EigenScore achieves AUROC~0.810 on SVHN and 0.880 on CIFAR-100;
EncMin2L outperforms on both by $+0.141$ and $+0.062$ respectively.
EigenScore's CIFAR-100 result (0.880) is the strongest among all pixel-space
diffusion methods surveyed, suggesting posterior covariance partially recovers semantic
structure---yet the residual $-0.062$ gap reflects class-conditional density geometry
that $\eta^2$ ($0.096$--$0.243$) quantifies and that pixel-space denoisers cannot fully access.

Under CIFAR-100 as ID, EigenScore deteriorates across all four benchmarks: CIFAR-10
$0.642$, SVHN $0.661$, DTD $0.627$, CelebA $0.427$.
Even DTD and CelebA lose the separation gap: a richer 100-class ID distribution
inflates the posterior covariance baseline, reducing spectral contrast for all OOD types.

\begin{remark}
\textit{Across all pixel-space diffusion methods---trajectory (DiffPath), inpainting
(LMD), projection (PR), LL+LID dual-threshold, and posterior covariance
(EigenScore)---a consistent asymmetry holds under both ID settings: strong on SVHN
(global domain shift), weak on semantic near-OOD.
Under CIFAR-100 as ID, three methods (DM Dual Threshold, LMD, Projection Regret)
exceed EncMin2L on SVHN ($0.945$/$0.985$/$0.945$ vs.\ $0.806$), yet all three
collapse on CIFAR-10 as OOD ($0.483$--$0.791$).
Pixel-space representations cannot access the class-conditional density geometry
($\eta^2\!=\!0.096$--$0.243$) that drives semantic discrimination in representation
space.}
\end{remark}

\begin{table*}[t]
  \caption{Leave-one-out encoder ablation for EncMin2L reporting
    AUROC (FPR@95\%TPR) on Near-OOD benchmarks. All configurations calibrated on the same held-out ID split.
    \textbf{Bold}: full combination.
    $\downarrow$: AUROC drop relative to full combination.}
  \label{tab:ablation}
  \centering
  \small
  \begin{tabular}{lccc}
    \toprule
    \textbf{Configuration} &
      \textbf{SVHN} &
      \textbf{CIFAR-100} &
      \textbf{CIFAR-10C} \\
    \midrule
    \textbf{E1+E2+E3 (full EncMin2L)}
      & \textbf{0.951 (0.301)} & \textbf{0.943 (0.257)} & \textbf{0.961 (0.212)} \\
    E2+E3 (no E1 / no CLIP)
      & 0.829 (0.494)~\color{blue}$\downarrow$0.122 \color{black} & 0.948 (0.225) & 0.961 (0.184) \\
    E1+E3 (no E2 / no DINOv2)
      & 0.958 (0.281) & 0.830 (0.512)~\color{blue}$\downarrow$0.113\color{black} & 0.972 (0.161) \\
    E1+E2 (no E3 / no ResNet)
      & 0.963 (0.236) & 0.953 (0.214) & 0.834 (0.727)~\color{blue}$\downarrow$0.127 \color{black}\\
    \bottomrule
  \end{tabular}
\end{table*}

\subsection{Encoder Contribution: Leave-One-Out Ablation}
\label{sec:ablation}

Table~\ref{tab:ablation} reports EncMin2L with each encoder removed in turn on the
three near-OOD benchmarks; DTD and CelebA are omitted as all configurations saturate
at $\geq\!0.999$.
Each drop directly confirms the specialisation the \S\ref{sec:diagnostics}
diagnostics predicted from ID data alone.

\textbf{Removing E1 (CLIP)} cuts SVHN AUROC by $-0.122$ while leaving
CIFAR-100 and CIFAR-10C stable: CLIP's low $\eta^2$ marks it as the
global-domain-shift specialist, with little contribution elsewhere.

\textbf{Removing E2 (DINOv2)} drops CIFAR-100 AUROC by $-0.113$; SVHN
and CIFAR-10C improve slightly, a complementarity effect explained by E2's
near-zero $\Delta\mu$---its corruption-invariant features were a drag on
covariate-shift benchmarks.
DINOv2's high $\eta^2\!=\!0.243$ identified this semantic specialisation
before any OOD data were observed.

\textbf{Removing E3 (ResNet-50)} causes the largest single drop:
CIFAR-10C falls by $-0.127$, while SVHN and CIFAR-100 improve slightly,
confirming that E3's large $\Delta\mu$ makes it indispensable for covariate
shift.

% ============================================================
\section{Discussion}
\label{sec:discussion}
% ============================================================

\paragraph{Encoder Generalization Advantage.}
The diagnostic pipeline applies to any pretrained encoder: compute $\eta^2$ and
$\Delta\mu$ on ID data, screen for complementary pairs via Spearman $\rho$, then
train per-encoder diffusion models only for encoders that add coverage.
A natural screening criterion is $\rho < 0.5$ between the candidate encoder and all
currently selected encoders; the remaining encoders each bring genuinely new
information.
This makes the framework a practical design-time tool for multi-encoder OOD
detector construction without ever needing OOD examples.

\paragraph{Recalibration-free Detection Advantage.}
The joint threshold $\tau$ is calibrated at FPR $\alpha\!=\!0.05$ on the ID
validation set.
Because the combined score $s(z)$ is approximately uniform on ID data
(Proposition~\ref{prop:beta}), changing $\alpha$ amounts to a direct rescaling of $\tau$:
the empirical FPR at any other $\alpha'$ can be read directly from the ID
validation CDF without recalibration.
This makes the system straightforwardly adjustable for deployment scenarios
with different false-positive budgets (e.g.\ conservative $\alpha\!=\!0.01$ for
safety-critical applications vs.\ permissive $\alpha\!=\!0.10$ for screening).

\paragraph{Limitations.}
The synthetic corruption probe ($\Delta\mu$) uses a fixed two-stage perturbation
(Gaussian blur + additive noise) to approximate covariate shift.
While this correctly ranks encoder sensitivity in our experiments, it may not
generalize to all covariate shift types (e.g.\ weather effects, sensor noise).
Users should run the probe with domain-relevant corruptions when the
expected shift type is known. The $\min(\cdot)$-gate's formal hypothesis testing guarantee (Tippett's test) requires the
per-encoder $p$-values to be independent under $H_0$.
In practice, Spearman correlations of 0.17--0.39 between encoder pairs
indicate moderate dependence. Near-OOD cases where OOD distribution partially overlaps with ID
distribution (e.g.\ CIFAR-10 ``truck'' vs.\ CIFAR-100 ``truck'' classes) remain
challenging for all density-based methods.

% ============================================================
\section{Conclusion}
\label{sec:conclusion}
% ============================================================

We presented a framework that statistically identifies the strengths and weaknesses
of pretrained encoders from ID data alone, and introduces EncMin2L---an
encoder-agnostic two-level $\min(\cdot)$-gate for combining and calibrating per-encoder
diffusion-based likelihood OOD detectors---achieving $\geq\!0.94$ AUROC across all
four shift types simultaneously and outperforming the state-of-the-art
diffusion OOD detectors at $2.3\!\times$
lower parameter cost.
$\eta^2$ and $\Delta\mu$ rank CIFAR-100 and CIFAR-10C AUROCs with Spearman
$\rho\!=\!0.771$ and $\rho\!=\!1.0$ respectively before any OOD examples are
observed, and leave-one-out ablations confirm that each encoder contributes precisely
where its diagnostic profile predicted.

% ============================================================
\section*{Acknowledgements}
\label{sec:conclusion}
% ============================================================
I would like to thank my colleagues Shaan Ul Haque and Nazal Mohamed, 
for their insightful feedback and discussions that greatly improved 
the quality of this work.

This research was supported in part through research cyberinfrastructure 
resources and services provided by the Partnership for an Advanced 
Computing Environment (PACE) at the Georgia Institute of Technology, 
Atlanta, Georgia, USA. RRID:SCR$\_$027619.

\bibliography{paper}
\bibliographystyle{abbrvnat}

\end{document}